\documentclass[journal]{IEEEtran}

\usepackage{color}

\newcommand{\argmin}{\operatornamewithlimits{arg\,min}}
\usepackage{booktabs}       % professional-quality tables
\usepackage{amsfonts}       % blackboard math symbols
\usepackage{nicefrac}       % compact symbols for 1/2, etc.
\usepackage{microtype}      % microtypography
\usepackage{bm}
\usepackage{amsmath}
\usepackage{amsthm}
\usepackage{amssymb,amsfonts}
\usepackage{graphicx}
\usepackage{subfigure}

\usepackage{multirow}
\usepackage{algorithm}
\usepackage{algorithmic}
\newtheorem{definition}{Definition}
\newtheorem{theorem}{Theorem}
\newtheorem{remark}{Remark}

\newtheorem{assumption}{Assumption}
\newtheorem{lemma}{Lemma}

\begin{document}
\title{Theoretical Analysis of Divide-and-Conquer ERM: Beyond Square Loss and RKHS}

\author{Yong Liu,
        Lizhong Ding
       % Xinwang Liu,
       % Lijun Zhang,
        %Shali Jiang,
       % Shizhong Liao
        and Weiping Wang
        \IEEEcompsocitemizethanks{
        \IEEEcompsocthanksitem
        Y. Liu is with Institute of Information
        Engineering, Chinese Academy of Sciences. Email: liuyong@iie.ac.cn.
        \IEEEcompsocthanksitem
        L.Z. Ding is with Inception Institute of Artificial Intelligence (IIAI), Abu Dhabi, UAE.
      %  \IEEEcompsocthanksitem
%        X.W. Liu is with National University of Defense Technology.
        % \IEEEcompsocthanksitem
%        L.J. Zhang is with National Key Laboratory for Novel Software Technology, Nanjing University.
%        Email: zljzju@gmail.com.
         %\IEEEcompsocthanksitem
%        S.L. Jiang is with Washington University.
        %\IEEEcompsocthanksitem
%        S.Z. Liao is with College of Intelligence and Computing, Tianjin University.
        \IEEEcompsocthanksitem
        W.P. Wang is with Institute of Information Engineering, Chinese Academy of Sciences.
        }
 %\thanks{Manuscript received Sep. 2019;
% }
}

\markboth{
%IEEE TRANSACTIONS ON INFORMATION THEORY
}%
{Shell \MakeLowercase{\textit{et al.}}: IEEE TRANSACTIONS ON INFORMATION THEORY
}

\maketitle

\begin{abstract}
    Theoretical analyses on the divide-and-conquer based distributed learning with least square loss in
the reproducing kernel Hilbert space (RKHS) have recently been explored within the framework of learning theory.
However, studies on learning theory for general loss functions and hypothesis spaces remain limited.
To fill this gap, we study the risk performance of distributed empirical risk minimization (ERM) for general loss functions and hypothesis spaces.
Our main contributions are two-fold.
First, we derive two tight risk bounds
under certain basic assumptions on the hypothesis space,
as well as the smoothness, Lipschitz continuity, and strong convexity of the loss function.
Second, we further develop a more general risk bound for distributed ERM without the restriction of strong convexity.
These results fill the gap in learning theory
of distributed ERM for general loss functions and hypothesis spaces.
\end{abstract}


\section{Introduction}
The rapid expansion in data size and complexity has introduced a series
of scientific challenges to the era of big data, such as storage bottlenecks and algorithmic scalability issues \cite{zhou2014big,Zhang2013,lin2017distributed}.
Distributed learning is the most popular approach for handling these challenges.
Among many strategies of distributed learning, the $divide\text{-}and\text{-}conquer$ approach has been shown most simple and effective,
while also being able to preserve data security and privacy by minimizing mutual information communications \cite{Zhang2013,zhang2015divide}.
This paper aims to study the theoretical performance of the $divide$-$and$-$conquer$ based distributed learning for
$Empirical~Risk~Minimization~(ERM)$ within a learning theory framework.
Given $$\mathcal{S}=\left\{z_i=(\mathbf  x_i,y_i)\right\}_{i=1}^N \in \left(\mathcal{Z}=\mathcal{X}\times \mathcal{Y}\right)^N,$$
drawn identically and independently (i.i.d) from an unknown probability  distribution $\mathbb{P}$ on
$\mathcal{Z}=\mathcal{X}\times\mathcal{Y}$,
the  ERM can be defined as
\begin{align}
\label{def-empirical-f}
  \hat{f}=\argmin_{f\in\mathcal{H}} \hat{R}(f):=\frac{1}{N}\sum_{j=1}^N\ell(f,z_j),
\end{align}
%and more generally
%\begin{align}
%\label{def-empirical-re}
%  \hat{f}=\argmin_{f\in\mathcal{H}} \hat{R}(f):=\frac{1}{N}\sum_{j=1}^N\ell(f,z_j)+r(f),
%\end{align}
where $\ell(f,z)$ is a loss function\footnote{If $\ell$ is a regularizer loss function, that is $\ell'(f,\cdot)=\ell(f,\cdot)+r(f)$,
$r(f)$ is a regularizer,
then \eqref{def-empirical-f} is related to a regularizer ERM.} and
%$r(f)$ is the regularizer,
$\mathcal{H}$ is a Hilbert space.
%$\|\cdot\|_\mathcal{H}$ is the norm in $\mathcal{H}$.
In distributed learning, the
data set $\mathcal{S}$ is partitioned into $m$ disjoint subsets $\{\mathcal{S}_i\}_{i=1}^m$,
and
$|\mathcal{S}_i|=\frac{N}{m}=:n.$
%Then it assigns each data subset $\mathcal{S}_i$ to one
%machine or processor to produce
The $i$th local estimator $\hat{f}_i$ is produced on each data subset $\mathcal{S}_i$:
\begin{align}
  \hat{f}_i=\argmin_{f\in\mathcal{H}}\hat{R}_i(f):=
    \frac{1}{|\mathcal{S}_i|}\sum_{z_j\in\mathcal{S}_i}\ell(f,z_j).
\end{align}
%Recall \cite{zhang2012communication}
The final global estimator $\bar{f}$ is then obtained by
$$\bar{f}=\frac{1}{m}\sum_{i=1}^m\hat{f}_i.$$
The theoretical foundations of  distributed learning for (regularized) ERM
have received increasing attention within the machine learning community,
and have recently been explored within the framework of learning theory
\cite{Zhang2013,lin2017distributed,li2013statistical,zhang2012communication,zhang2015divide,xu2016feasibility,mucke2018parallelizing,Lin2018dk,chang2017divide}.
However, most existing risk analyses are based on the
closed form of the least square solution and the properties of the  reproducing kernel Hilbert space (RKHS),
which is only suitable when the distributed learning uses a least square loss in the RKHS.
Studies on establishing the risk bounds of distributed learning for
 general loss functions and hypothesis spaces remain limited.

%thus it cannot be easily extended to other non-square loss functions.
In this paper, we study the risk performance of distributed ERM based on the divide-and-conquer approach for general loss functions and hypothesis spaces.
Specifically, we use the proof techniques from stochastic convex optimization for general loss functions
and the covering number for general hypothesis space.
Note that the proof techniques of stochastic convex optimization and covering numbers are usually two significantly different paths for theoretical analysis.
The main technical difficulty of this paper is thus integrating these two different proof techniques for distributed learning.

The main contributions of the paper include:
\begin{itemize}
  \item \textbf{Result I}.
  If the number of processors $m\leq O\big(\sqrt{N}\big)$,
  we present a tight risk bound with order\footnote{We use $\tilde{\Omega}$ and $\tilde{O}$ to hide constant factors
  as well as polylogarithmic factors in $N$ or $m$.} $O\left(\frac{h}{N}\right)$,
  assuming there is a logarithmic covering number of hypothesis space ($\log C(\mathcal{H},\epsilon)\simeq h\log\frac{1}{\epsilon}$,
  $h>0$, see Assumption \ref{log-metric} for details),
  and a smooth, Lipschitz continuous and strongly convex loss function.
  \item \textbf{Result II}. Under another basic assumption that there exists a hypothesis space of polynomial covering number ($\log C(\mathcal{H},\epsilon)\simeq \left(\frac{1}{\epsilon}\right)^{\frac{1}{h}}$,
  see Assumption \ref{poly-metric} for details), and if the number of processors $m\leq O\big(N^{\frac{h}{2h+1}}\big)$,
  %and the smoothness, Lipschitz continuous and strongly convex of loss function,
  another tight risk bound of order $O\big(N^{-\frac{2h}{2h+1}}\big)$ is established.
  \item \textbf{Result III}. Without Result I's restriction of a strong convexity of loss function,
  a more general risk bound of order $O\left(\frac{h}{N^{1-r}}\right)$ is derived when the number of processors $m\leq O\left(N^{r}\right)$, $0 \leq r< \frac{1}{2}$, and the optimal risk is small.
  Since $0 \leq r< \frac{1}{2}$, the rate is faster than $O\big(\frac{1}{\sqrt{N}}\big)$.
  %When $r\rightarrow 0$,
%  the risk reaches the optimal rate.
  %\item \textbf{Result IV}. Without the restriction of a  strongly convex loss in Result II,
%  a more general risk bound of order $O\left({N^{-\frac{h}{h+1}+r}}\right)$ is established when $m\leq O\left(N^{r}\right)$,
%  $0 \leq r<\frac{h}{h+1}$.
%  %If $r\rightarrow 0$ and $h\rightarrow \infty$, the risk is optimal.
\end{itemize}
Overall,
these results fill the gap in learning theory of distributed ERM for general loss functions and hypothesis spaces.
%which can be used to understand the risk performance for general distributed learning machines.

%our results are not dependent on the least square solution and  RKHS,
%thus are not only suit to kernel-based regression with square loss,
%but also to other learning machines.
%which generalizes the results of \cite{lin2017distributed,Zhang2013,zhang2015divide}.

The rest of the paper is organized as follows.
In Section 2, we discuss our main results.
In Section 3, we compare against related work.
Section 4 is the conclusion.
All the proofs are given in the last part.

\section{Main Results}
In this section, we provide and discuss our main results.
To this end, we first introduce several notations.
\begin{definition}[$\epsilon$-covering \cite{zhou2002covering}]
    Let $(\mathcal{H},\|\cdot\|_\mathcal{H})$ be a Hilbert space,
    $\mathcal{N}(\mathcal{H},\epsilon)$ is an $\epsilon$-covering of the Hilbert space $\mathcal{H}$
    if $\forall f\in \mathcal{H}$,
     $\exists\tilde{f}\in \mathcal{N}(\mathcal{H},\epsilon)$ such that
    $$\|f-\tilde{f}\|_\mathcal{H}\leq\epsilon.$$
\end{definition}
\begin{definition}[$\epsilon$-covering number \cite{zhou2002covering}]
    Let $C(\mathcal{H},\epsilon)$ be the $\epsilon$-covering number of $\mathcal{H}$,
    that is, the smallest number of cardinality for an $\epsilon$-covering of $\mathcal{H}$ which can written as
    \begin{align*}
      C(\mathcal{H},\epsilon):=\min\left\{k: \exists \epsilon\text{-covering over~}\mathcal{H}\text{ of size }k\right\}.
    \end{align*}
\end{definition}
%Let $\mathcal{N}(\mathcal{H},\epsilon)$ be an $\epsilon$-covering of a hypothesis space $\mathcal{H}$, i.e.,

The performance of a function $f$ given by the learning machine is usually measured by
the $risk$
\begin{align*}
  R(f):=\mathbb{E}_{z}[\ell(f,z)].
\end{align*}
%which is the most popular used measure to the learning performance of the function $f$.
%Let  $R(f):=\mathbb{E}_{z}[\ell(f,z)]$ be the risk of $f$.
We denote the optimal function and  risk of $\mathcal{H}$, respectively, as
\begin{align*}
  f_\ast:=\argmin_{f\in\mathcal{H}}R(f) \text{~and~} R_\ast:&=R(f_\ast).
\end{align*}
%$$R(f):=\mathbb{E}_{z}[\ell(f,z)], f_\ast:=\argmin_{f\in\mathcal{H}}R(f)\text{ and }R_\ast=R(f_\ast).$$
%$R(f)$ is the risk of $f$, $f_\ast$ is the optimal function in $\mathcal{H}$.
%The optimal risk of $\mathcal{H}$ is denoted as $R_\ast=R(f_\ast)=\mathbb{E}_{z}[\ell(f_\ast,z)]$.
\subsection{Assumptions}
In this subsection, we introduce some basic assumptions of the hypothesis space and loss function.
%\textbf{Assumptions of hypothesis spaces}
\begin{assumption}[\textbf{logarithmic covering number}]
\label{log-metric}
  There exists some $h>0$ such that
  \begin{align}
  \label{ass-1-eq}
    \forall\epsilon\in(0,1), \log C(\mathcal{H},\epsilon)\simeq h\log(1/\epsilon).
  \end{align}
\end{assumption}
Many popular function classes satisfy the above assumption when the hypothesis $\mathcal{H}$ is bounded:
\begin{itemize}
  \item Any function space with finite VC-dimension \cite{vaart1996weak},
              including linear functions and univariate polynomials of degree $k$ (for which $h=k+1$) as special cases;
  \item Any RKHS based on a kernel with rank $h$ \cite{carl1980inequalities}.
\end{itemize}
%such as linear functions (for which $h=d$, $d$ is the dimension of input $\mathcal{X}$),
%univariate polynomials of degree $k$ (for which $h=k+1$), and more generally,
%any function space with finite VC-dimension \cite{vaart1996weak}.
%This type of scaling also holds for a.

\begin{assumption}[\textbf{polynomial covering number}]
\label{poly-metric}
  There exists some $h>0$ such that
  \begin{align}
  \label{eq-ass-ploy}
    \forall\epsilon\in(0,1), \log C(\mathcal{H},\epsilon)\simeq (1/\epsilon)^{1/h}.
  \end{align}
\end{assumption}
If $\mathcal{H}$ is bounded, this type of covering number is satisfied
by many Sobolev/Besov classes \cite{gu2013smoothing}.
For instance, if the kernel eigenvalues decay at a rate of $k^{-2h}$,
then the  RKHS satisfies Assumption  \ref{poly-metric} \cite{carl1980inequalities}.
For the RKHS of a Gaussian kernel, the kernel eigenvalues decay at a rate of $h\rightarrow \infty$.
%\subsubsection{Assumptions of loss function}

\begin{remark}
To derive the risk bounds for divide-and-conquer ERM without specific assumptions on the type of hypothesis,
we adopt the covering number to measure the complexity of the hypothesis.
To use the covering number in learning theory,
an assumption on the bounded hypothesis is usually needed (see \cite{carl1980inequalities,gu2013smoothing} for details).
In fact, ERM usually includes a regularizer, that is
$$
  \min_{f\in \mathcal{H}} \frac{1}{N}\sum_{j=1}^N\ell(f,z_i)+\lambda\|f\|_\mathcal{H}^2,
$$
which is equivalent to the following optimization for a constant $c$ related to $\lambda$,
$$
  \min_{f\in \mathcal{H}} \frac{1}{N}\sum_{j=1}^N\ell(f,z_i), \text{s.t. }\|f\|_\mathcal{H}^2\leq c.
$$
Thus, the assumption for the bounded hypothesis is usually implied in (regularized) ERM.
\end{remark}

\begin{assumption}
\label{assumption-G-L}
  The loss function  $\ell(f,z)$ is non-negative, $G$-smooth, $L$-Lipschitz
  continuous, and convex w.r.t $f$ for any $z\in\mathcal{Z}$.
\end{assumption}
Assumption \ref{assumption-G-L}  is satisfied by several popular losses when $\mathcal{H}$ and $\mathcal{Y}$ are bounded,
such as the square loss $\ell(f,z)=(f(\mathbf x)-y)^2$,
logistic loss $\ell(f,z)=\ln(1+\exp(-yf(\mathbf x)))$,
square Hinge loss $\ell(f,z)=\max(1-yf(\mathbf x))^2$,
square $\epsilon$-loss $\ell(f,z)=\max(0,|y-f(\mathbf x)|-\epsilon)^2$,
and so on.

%Assumption \ref{assumption-G-L} allows us to model some popular losses,
%such as square loss $\ell(f,z)=(f(\mathbf x)-y)^2$,
%logistic loss $\ell(f,z)=\ln(1+\exp(-yf(\mathbf x)))$,
%square Hinge loss $\ell(f,z)=\max(1-yf(\mathbf x))^2$,
%square $\epsilon$-loss $\ell(f,z)=\max(0,|y-f(\mathbf x)|-\epsilon)^2$,
%and so on.

\begin{assumption}
\label{assumption-strongly-loss}
  The loss function  $\ell(f,z)$ is an $\eta$-strongly convex function  w.r.t $f$ for any $z\in\mathcal{Z}$.
\end{assumption}
%Assumption \ref{assumption-strongly-loss} also allows us to model serval popular losses,
%such as square loss $\ell(f,z)=(f(\mathbf x)-y)^2$,
%logistic loss $\ell(f,z)=\ln(1+\exp(-yf(\mathbf x)))$,
%and so on.
Note that $\ell(f,\cdot)$ usually includes a regularizer,
e.g. $\ell(f,\cdot)=\tilde{\ell}(f,\cdot)+\eta\|f\|^2_\mathcal{H}$.
In this case,
$\ell(f,\cdot)$ is a strongly convex function which only requires $\tilde{\ell}(f,\cdot)$ to be a convex function.

\begin{assumption}[\textbf{$\tau$-diversity}]
\label{assumption-diversity}
There exists some $\tau>0$ such that
 %%Let $\Delta_{\bar{f}}=\frac{1}{m^2}\sum_{i,j=1,i\not=j}^m\|\hat{f}_i-\hat{f}_j\|_\mathcal{H}^2$ be diversity between all partition-based estimates.
% Let $\Delta_{\bar{f}}$ be the diversity between all partition-based estimates,
 %$\Delta_{\bar{f}}=\frac{1}{4m^2}\sum_{i,j=1,i\not=j}^m\|\hat{f}_i-\hat{f}_j\|_\mathcal{H}^2$.
 %The distributed algorithm is $\tau$-diversity, $\tau>0$,
% that is
 \begin{align*}
    \frac{1}{4m^2}\sum_{i,j=1,i\not=j}^m\|\hat{f}_i-\hat{f}_j\|_\mathcal{H}^2\geq \tau,
 \end{align*}
 where $\frac{1}{4m^2}\sum_{i,j=1,i\not=j}^m\|\hat{f}_i-\hat{f}_j\|_\mathcal{H}^2$ is the diversity between all partition-based estimates,
 and $\hat{f}_i$ is the $i$th local estimator, $i=1,\ldots, m$.
\end{assumption}
If not all the partition-based estimates $\hat{f}_i$, $i=1,\ldots, m$, are almost the same,
Assumption \ref{assumption-diversity} is satisfied.
\subsection{Risk Bounds}
In the following,
we first derive two tight risk bounds with a smooth, Lipschitz continuous and strongly convex function.
Then, we further consider the more general case by removing the restriction of strong convexity.
\begin{theorem}
  \label{theorem-log}
  Under Assumptions \ref{log-metric}, \ref{assumption-G-L}, \ref{assumption-strongly-loss}, \ref{assumption-diversity},
  if the number of processors $m$ satisfies the bound:
  \begin{align*}
    m\leq \min\left\{
    \frac{N\eta}{8Gh\log \frac{N}{\delta}}, \frac{\sqrt{Nh\eta}}{L\log\frac{2}{\delta}}, \frac{h\eta+N\eta^2\tau}{128GR_\ast\log\frac{2}{\delta}}, \frac{N\eta}{GL\log\frac{2N}{\delta}}\right\},
  \end{align*}
  then,  with probability at least $1-\delta$,
  we have\footnote{In this paper, polylogarithmic factors are usually ignored or considered constant
for simplicity. The $\log^2$  is usually written as $\log$ for simplicity.}
  \begin{align*}
   %\label{finial-dd}
   R(\bar{f})-R(f_\ast)\leq  O\left(\frac{h}{N}\right).
  \end{align*}
  %where $\Delta_{\bar{f}}=\frac{1}{4m^2}\sum_{i,j=1,i\not=j}^m\|\hat{f}_i-\hat{f}_j\|_\mathcal{H}^2$.
\end{theorem}
%%\begin{remark}
%%\label{remark-one}
%  %The constant in big $O$ can be seen from the proof.
%  %which is tedious and included in the supplement.
  The above theorem implies that when $\ell$ is smooth, Lipschitz continuous and strongly convex,
  the distributed ERM achieves a risk bound in the order of
 $R(\bar{f})-R(f_\ast)=O\left(\frac{h}{N}\right).$
% % under the restriction $m\leq \tilde{O}(\sqrt{N})$.
%  %Under the assumption \ref{log-metric} (logarithmic covering number),
 This rate in Theorem \ref{theorem-log} is \textbf{minimax-optimal} for some cases:
  \begin{itemize}
  \item \textbf{Finite VC-dimension}. If the VC-dimension of $\mathcal{H}$ is bounded by $h$,
  which is a special case of Assumption \ref{log-metric},
  \cite{ehrenfeucht1989general,zhivotovskiy2016localization,hanneke2016refined}
  showed that there exists a constant $c'\geq 0$ and a function $f\in \mathcal{H}$,
   such that $$R(f)-R(f_\ast)\geq c'\frac{h}{N}.$$
  \item \textbf{Square loss}. Note that, for the square loss function, $R(f)-R(f_\ast)=\mathbb{E}_{z}[\|f-f_\ast\|_2^2]$.
  From Theorem 2(a) of \cite{Raskutti2012}  with $s=d=1$,
  we find that, under  Assumption \ref{log-metric},
  there is a universal constant $c'>0$ such that $$\inf_f\sup_{f_\ast\in \mathbb{B}_\mathcal{H}(1)}[R(f)-R(f_\ast)]\geq c'\frac{h}{N},$$
  where $\mathbb{B}_\mathcal{H}(1)$ is  the 1-norm ball in $\mathcal{H}$.
  \end{itemize}

   From Theorem \ref{theorem-log},
  we know that, to achieve  the tight risk bound, the number of processors $m$ should satisfy the restriction
  \begin{align*}
    m\leq  \Omega\left(\min\left\{
    \frac{N\eta}{\log \frac{N}{\delta}}, \frac{\sqrt{N}}{\log{\frac{1}{\delta}}}, \frac{N}{R_\ast\log\frac{1}{\delta}}, \frac{N}{\log\frac{N}{\delta}}
    \right\}\right).
  \end{align*}
  Thus, $m$ can reach
  $
   \tilde{\Omega}\big(\sqrt{N}\big),
  $
 which is sufficient for using distributed learning in practical applications.

\begin{theorem}
  \label{theorem-poly}
  Under Assumptions \ref{poly-metric}, \ref{assumption-G-L}, \ref{assumption-strongly-loss}, \ref{assumption-diversity},
  if the number of  processors $m$ satisfies the bound:
  \begin{align*}
    m\leq \min&\left\{\frac{N\eta}{4\left(N^{\frac{1}{2h+1}}+\log (2/\delta)\right)},\frac{\sqrt{\eta}N^{\frac{h+1}{2h+1}}}{L\log(2/\delta)},\right.
    \\ &~~~\left.\frac{\eta N^{\frac{h}{2h+1}}}{GL\log(2/\delta)},
  \frac{\eta N^{\frac{1}{2h+1}}+N\eta^2\tau}{128 R_\ast \log(2/\delta)}\right\},
  \end{align*}
  then,  with probability at least $1-\delta$,
  we have
  \begin{align*}
    R(\bar{f})-R(f_\ast)\leq O\left(N^{-\frac{2h}{2h+1}}\right).
  \end{align*}
\end{theorem}
  %The above theorem implies that the distributed  ERM achieves a risk bound of order
%  $$  R(\bar{f})-R(f_\ast)=O\left(N^{-\frac{2h}{2h+1}}\right).$$
 % \begin{align}
%  \label{equaiton-delta}
%  \begin{aligned}
%    m\leq \left\{\begin{array}{ll}
%               \tilde{O}\Big(N^{\frac{h+1}{2h+1}}\Big), & h\geq \sqrt{2}, \\
%               \tilde{O}\Big(N^{\frac{2h^2-1}{(2h+1)h}}\Big), & \frac{\sqrt{2}}{2}\leq h\leq \sqrt{2}.
%             \end{array}
%    \right.
%  \end{aligned}
%  \end{align}
  From Theorem 2(b) of \cite{Raskutti2012}  with $s=d=1$,
  we know that,  under Assumption \ref{poly-metric},
  there is a universal constant $c'>0$ such
  that
  \begin{align*}
    \inf_f\sup_{f_\ast\in \mathbb{B}_\mathcal{H}(1)}\mathbb{E}_{z}[\|f-f_\ast\|_2^2]\geq c'N^{-\frac{2h}{2h+1}}.
  \end{align*}
  Thus, our risk bound of order $O\big(N^{-\frac{2h}{2h+1}}\big)$ is minimax-optimal in this case.

   From Theorem \ref{theorem-poly}, we know that,
   to achieve  the tight risk bound, the number of processors $m$ should satisfy the restriction
   \begin{align*}
    m &\leq \tilde{\Omega}\left(\min\left\{N^{\frac{2h}{2h+1}},N^{\frac{h+1}{2h+1}},N^{\frac{h}{2h+1}}\right\}\right)\\
    &=\tilde{\Omega}\left(N^{\frac{h}{2h+1}}\right).
  \end{align*}
  Note that
  $\frac{h}{2h+1}\leq \frac{1}{2}$,
  thus the number of processors $m\leq \tilde{\Omega}(\sqrt{N})$, which is smaller than in Theorem \ref{theorem-log}.
  This is because the restriction of the polynomial covering number is looser than that of the logarithmic one.
  When $h\rightarrow \infty$ (satisfied by the Gaussian kernel),
  $m$ can reach $\tilde{\Omega}\big({\sqrt{N}}\big)$.
%  the relaxation of the loss function restriction
  %Note that when $h\geq 1$,
%  \begin{align*}
%   \frac{h-1}{h}\leq \frac{h+1}{h+2} \text{ and } \frac{h-1}{h} \leq \frac{2h^2-1}{h(2h+1)}
%  \end{align*}
%  $\frac{h-1}{h}\leq \frac{h+1}{h+2}$ and $\frac{h-1}{h}\leq \frac{2h^2-1}{h(2h+1)}$

\begin{table*}[t]
%\small
%\footnotesize
%\scriptsize
%\tiny
%\renewcommand{\captionfont}{\scriptsize}
   \caption{ Summary of Risk Bounds of Distributed ERM.
   }
   \label{tabel:risk bound}
    %\centering
   % \renewcommand{\multirowsetup}{\centering}
  % \scriptsize
  % \tiny
    \begin{tabular*}{\linewidth}{@{\extracolsep{\fill}}cc|c|c|c|c|c}
    \toprule
          Paper   & Loss Function &~Hypothesis Space & Other Condition   & Risk Bound & Optimal &Partitions \\ \hline
          \multirow{2}{*}{\cite{zhang2015divide}} & Square loss & Assumption \ref{log-metric} &  Eigenfunctions \eqref{log-metric} & $O\Big(\frac{h}{N}\Big)$ &Yes &$\Omega\Big(N^{\frac{k-4}{k-2}}\Big)$ \\ \cline{2-7}
                                                  & Square loss & Assumption \ref{poly-metric} &  Eigenfunctions \eqref{log-metric} &$O\Big(N^{-\frac{2h}{2h+1}}\Big)$ &Yes &$\Omega\Big(N^{\frac{2(k-4)h-k}{2h+1}}\Big)$ \\ \hline
           \cite{lin2017distributed}              &Square loss & RKHS  &Regularity condition \eqref{con-regu} &$O\left(N^{-\frac{2s}{2s+1}}\right)$ &Yes &$\Omega(1)$ \\ \hline
           \cite{chang2017distributed}            &Square loss &RKHS   &Regularity condition \eqref{con-regu} &$O\left(N^{-\frac{2s}{2s+1}}\right)$ &Yes &$\Omega\Big(N^{\frac{2s-1}{2s+1}}\Big)$\\ \hline
            Theorem \ref{theorem-log}             &Assumptions \ref{assumption-G-L}, \ref{assumption-strongly-loss} & Assumption \ref{log-metric} &Assumption \ref{assumption-diversity} &$O\left(\frac{h}{N}\right)$
            &Yes & $\tilde{\Omega}\left(\sqrt{N}\right)$\\ \hline
            Theorem \ref{theorem-poly} &Assumptions \ref{assumption-G-L}, \ref{assumption-strongly-loss} &Assumption \ref{poly-metric} &Assumption \ref{assumption-diversity} &$O\Big(N^{-\frac{2h}{2h+1}}\Big)$ &Yes &$\tilde{\Omega}\Big(N^\frac{h}{2h+1}\Big)$\\ \hline
            Theorem \ref{the-non-strong} & Assumptions \ref{assumption-G-L} &  Assumption \ref{log-metric} &-- &$\tilde{O}\Big(N^{-(1-r)}\Big)$ &Yes if $r\rightarrow 0$&$\Omega(N^r),r\in[0,1/2]$ \\
       %     Theorem \ref{the-non-strong-ploy} &Assumptions \ref{assumption-G-L} &Assumption \ref{poly-metric} &--&$\tilde{O}\Big(N^{-\frac{h}{h+1}+r}\Big)$&Yes if $r\rightarrow 0$, $h\rightarrow \infty$ &$O(N^r),r\in\left[0,\frac{h}{h+1}\right]$ \\
    \bottomrule
\end{tabular*}
%\vspace{-0.8cm}
\end{table*}
\subsection{Risk Bounds without Strong Convexity}
As follows, we provide a more general risk bound
without the restriction of strong convexity.
\begin{theorem}
  \label{the-non-strong}
  Under Assumptions \ref{log-metric}, \ref{assumption-G-L}
  and assuming that $\forall f\in \mathcal{H}$, $\|f\|_\mathcal{H}\leq B$,
  if the number of processors $m\leq O(N^{r})$, $0 \leq r\leq \frac{1}{2}$,
  then, with a probability of at least $1-\delta$,
  we have
  \begin{align*}
   R(\bar{f})-R(f_\ast)\leq
     O\left(
     \frac{h\log({N}/{\delta})}{N^{1-r}}+\sqrt{\frac{R_\ast \log\frac{1}{\delta}}{N^{1-r}}}
     \right).
  \end{align*}
  If the optimal risk $R_\ast$ is small, that is $R_\ast\leq O(N^{r-1})$,
  we have
  \begin{align*}
    R(\bar{f})-R(f_\ast)\leq
     O\left(
     \frac{h\log({N}/{\delta})}{N^{1-r}}
     \right).
  \end{align*}
\end{theorem}
From the above theorem, one can see that:
\begin{itemize}
\item[1)] A general risk bound without the restriction of strong convexity is established.
The rate of this theorem is $\tilde{O}\Big(\frac{h}{N^{1-r}}+\sqrt{\frac{R_\ast}{N^{1-r}}}\Big)$,
which is worse than that of Theorem \ref{theorem-log}.
This is due to the relaxation of the loss function restriction.
\item[2)] The above theorem implies that, when the optimal risk $R_\ast$ is small,
the risk bound is in the order of $\tilde{O}\left(\frac{1}{N^{1-r}}\right)$.
Note that $0\leq r \leq \frac{1}{2}$, so in this case,
the rate is faster than $O\Big(\frac{1}{\sqrt{N}}\Big).$
\item[3)] In the central case, that is $m=1$,
the order of the risk can reach
\begin{align*}
  R(\bar{f})-R(f_\ast)= \tilde{O}\left(\frac{h}{N}\right),
\end{align*}
which is nearly optimal.
To the best of our knowledge,
such a fast rate of ERM for the central case  has never been given before for general loss function and  hypothesis space.
%that is the number of processors $m$ does not increase with the sample size $N$,
%we obtain that

%which is an optimal rate following a similar analysis as Theorem \ref{theorem-log}.
\end{itemize}
\begin{remark}
  In Theorem \ref{the-non-strong}, the risk bound is satisfied for all $m\leq O(N^r)$, $r\in [0, 1/2]$.
Parameter $r$ is used to balance the tightness of the bound and number of processors.
The smaller the $r$, the tighter the risk bound and the fewer the processors.
\end{remark}

\section{Comparison with Related Work}
In this section, we compare our results with related work.
%\textbf{Risk analysis for original ERM}.
Risk analyses for the original (regularized) ERM have been extensively explored within the framework of learning theory
\cite{vapnik1998naturestatistical,bartlett2003rademacher,de2005model,Bartlett2005Local,Ding2014mscn,caponnetto2007optimal,steinwart2009optimal,smale2007learning,steinwart2008support,Zhang2017er}.
Recently,
divide-and-conquer based distributed learning  with ridge regression \cite{li2013statistical,Zhang2013,zhang2015divide,xu2016feasibility},
gradient descent algorithms \cite{shamir2014distributed,Lin2018dk}, online learning \cite{yan2013distributed},
local average regression \cite{chang2017divide},
spectral algorithms \cite{guo2017learning,guo2017thresholded},
%bias correct regularized kernel network \cite{guo2017learning},
semi-supervised learning \cite{chang2017distributed} and
the minimum error entropy principle \cite{hu2019distributed},
have been proposed and their learning performances have been observed in many practical applications.
%to be as good as that of a big processor which could handle the whole data.
For point estimation, \cite{li2013statistical}
showed that the distributed moment
estimation is consistent if an unbiased estimate is obtained for each of the subproblems.
For the distributed regularized least square in RKHS,
\cite{xu2016feasibility} showed that
the distributed ERM leads to an estimator that is
consistent with the unknown regression function.
%the mean-squared error decays as
%$\mathbb{E}\big[\big\|\bar{f}-f_\ast\big\|^2\big]=\mathcal{O}\big(\frac{1}{\sqrt{Nm}}\big)$ under some conditions,
%where $f_\ast$ is the optimal hypothesis in the hypothesis space.
Under local strong convexity, smoothness and a reasonable set of other conditions,
an improved bound was established in \cite{zhang2012communication}.

Optimal learning rates for divide-and-conquer kernel ridge regression in expectation
were established in \cite{zhang2015divide},
under certain eigenfunction assumptions.
Removing these assumptions,
an improved bound was derived in \cite{lin2017distributed}
using a novel integral operator method.
Using similar proof techniques as \cite{lin2017distributed} or \cite{zhang2015divide},
optimal learning rates were established for distributed spectral algorithms \cite{guo2017learning},
kernel-based distributed gradient descent algorithms \cite{Lin2018dk},
kernel-based distributed semi-supervised learning \cite{chang2017distributed},
distributed local average regression \cite{chang2017divide},
and distributed KRR with communications \cite{lin2020distributed}.

Among these works, \cite{zhang2015divide,lin2017distributed,chang2017distributed} are the three  most relevant papers.
%Unfortunately, the optimal learning rates for these distributed learning methods
%depend on the special properties of the square loss and RKHS (such as their closed form, and the integral operator of the kernel function),
%which do not apply when analyzing the performance under other loss functions and hypothesis spaces.
%To address this issue,
%in this paper, we derive the risk bounds based on the general properties of loss functions and hypothesis space,
%making them more generalizable.
%
%
%The theoretical foundations of  distributed learning for (regularized) ERM
%have recently been explored  within a learning theory framework
%\cite{li2013statistical,zhang2012communication,zhang2015divide,xu2016feasibility,lin2017distributed,mucke2018parallelizing,Lin2018dk,chang2017divide}.
Thus, as follows, we will compare our results with those in  \cite{zhang2015divide,lin2017distributed,chang2017distributed}.
The seminal work of \cite{zhang2015divide} considered the learning
  performance of divide-and-conquer kernel ridge regression.
  Using a matrix decomposition approach,
  \cite{zhang2015divide} derived two optimal learning rates of order $O(\frac{h}{N})$ and
  $O\big(N^{-\frac{2h}{2h+1}}\big)$, respectively,
  for the $h$-finite-rank kernels and $h$ polynomial eigen-decay kernels,
  under the assumption that, for some constants $k\geq 2$ and $A<\infty$,
  the normalized eigenfunctions $\{\phi_i\}_\ell$ satisfy
  \begin{align}
  \label{eq-eigenvae-f}
    \forall j=1,2,\ldots, \mathbb{E}[\phi_j(X)^{2k}]\leq A^{2k}.
  \end{align}
  The condition in \eqref{eq-eigenvae-f} is possibly too strong,
  and it was thus removed in \cite{lin2017distributed},
  which used a novel integral operator approach under the regularity condition:
  \begin{align}
  \label{con-regu}
    f_\rho=L_K^s h_\rho, \text{~for~some $0<s\leq 1$ and $h_\rho\in L^2_\rho$},
  \end{align}
  where $L_K$ is the integral operator induced by the kernel function $K$:
  \begin{align*}
    L_K(f)(\mathbf x):=\int_\mathcal{X} K(\mathbf x,\mathbf x')f(\mathbf x') d\mathbb{P}_\mathcal{X}(\mathbf x') ,\mathbf x\in\mathcal{X},
  \end{align*}
   and $f_\rho=\int_\mathcal{Y} yd \mathbb{P}(y|\mathbf x)$ is the regression function.
  However, the analysis in \cite{lin2017distributed} only works for $s>1/2$.
  In \cite{chang2017distributed}, they generalized the results of \cite{lin2017distributed},
  and derived the optimal learning rate for all $1/2\leq s\leq 1$ under the restriction $m\leq N^{\frac{2s-1}{2s+1}}$ for bounded kernel functions.
 Thus, we find that, for the special case of $s=1/2$,
 the number of local processors $m\rightarrow\Omega(1)$, does not increase with $N$.
Note that $1/2\leq  s\leq 1$, so the largest number of local processors can only reach $m=\Omega(N^{1/3})$,
which may limit the applicability of distributed learning.

Our and the most related previous results are summarized in Table \ref{tabel:risk bound}.
Compared with previous works, there are two main novelties of our results.
\begin{itemize}
\item[1)] The proof techniques of this paper are based on
the general properties of loss functions and hypothesis spaces,
while for \cite{zhang2015divide,lin2017distributed,chang2017distributed},
the proofs depend on the special properties of the square loss and RKHS.
Thus, our results generalized the results of \cite{zhang2015divide,lin2017distributed,chang2017distributed};
\item[2)] To derive the optimal rates, \cite{lin2017distributed,chang2017distributed}
  show that the number of local processors should be less than $\Omega\big(N^{\frac{2s-1}{2s+1}}\big)$, $1/2\leq s\leq 1$.
  Thus, the highest number $m$ will be restricted by a constant for $s=1/2$,
  and the best result is $\Omega(N^{1/3})$ (for $s=1$).
  However, in this paper,
the number of processors that our result can reach is ${\Omega}(\sqrt{N})$.
%For some cases, $m$ can even reach  $\tilde{\Omega}(N)$.
Thus, our result can relax the restriction on the number of processors
form \cite{lin2017distributed,chang2017distributed}.
\end{itemize}

\section{Conclusion}
In this paper, 
we studied the risk performance of the divide-and-conquer ERM
and derived tight risk bounds for general loss functions and hypothesis spaces.
To make our results suitable for general loss functions and hypothesis spaces,
we used the proof techniques from stochastic convex optimization and the covering number,
which are usually two significantly different paths for theoretical analysis.
These results fill the gap in learning theory
of distributed ERM,
and the proof techniques we used may provide a new path for theoretical analysis.

There are some work worth studying in the future.
\begin{itemize}
  \item[1)] In our analysis, we assume that the loss function is a (strong) convex function.
  How to extend our results to a non-convex function is an interesting direction. 
  %In the future, we will consider to use the Polyak-Lojasiewicz condition  \cite{karimi2016linear} instead of a (strong) convex function.
  \item[2)] In \cite{chang2017distributed}, they showed that the number of processors can be improved using the unlabeled samples.
  To use the unlabeled samples to improve our results may be a good question. 
  %In the future, we will consider using the unlabeled samples to improve our results.
  \item[3)] In this paper, we only considered the simple divide-and-conquer based distributed learning.
  To extend our results to other distributed learning scenarios, such as distributed KRR with communications \cite{lin2020distributed}, is worthy of attention.
\end{itemize}

%hypothesis space,
%which
%
%
%based on the proof techniques stochastic convex optimization and covering numbers,
%which
%
%the two significantly different paths
%%Note that the proof techniques of stochastic
%%convex optimization and covering numbers are usually two
%%significantly different paths for theoretical analysis
%
%we use the proof techniques from stochastic convex optimization and the covering number

%We first showed that, when the number of processors satisfies certain restrictions,
%we can obtain tight risk bounds, assuming  there is a logarithmic (or polynomial) covering number for the hypothesis space,
%and a smooth, Lipschitz continuous and strongly
%convex loss function.
%We further presented a more general risk bound by removing the restriction of strong convexity.
%Our results fill the gap in learning theory
%of distributed ERM for general loss functions and hypothesis spaces.
%extend the analysis scope of distributed ERM learning theory for general learning machines,

\section*{Acknowledgments}
This work was supported in part by
the National Natural Science Foundation of China (No.61703396, No.61673293),
the CCF-Tencent Open Fund,
the Youth Innovation Promotion Association CAS,
the Excellent Talent Introduction of Institute of Information Engineering of CAS (No. Y7Z0111107),
the Beijing Municipal Science and Technology Project (No. Z191100007119002),
and the Key Research Program of Frontier Sciences, CAS (No. ZDBS-LY-7024).
\section{Proof}
In this section, we first introduce the key idea of proof,
and then provide proofs for Theorems \ref{theorem-log}, \ref{theorem-poly} and \ref{the-non-strong}.
%and then give the proof of Theorem  \ref{theorem-log}.
%The other proofs are given in the supplementary material.
%%Let $R(f)=\mathbb{E}_{z\sim\mathbb{P}}[\ell(f,z)]$, $\hat{R}_i(f)=\frac{1}{|\mathcal{S}_i|}\sum_{z_j\in\mathcal{S}_i}\ell(f,z_j)$.
\subsection{The Key Idea}
%We will give the key idea of the proof.
Note that if $\ell$ is an $\eta$-strongly convex function,
then, $R(f)$ is also $\eta$-strongly convex.
According to the properties of a strongly convex function, $\forall f,f'\in\mathcal{H}$,
we have
  \begin{align}
  %\begin{aligned}
    \label{assumption-strongly-equation}
     &~~~\left\langle \nabla R(f'), f-f'\right\rangle_\mathcal{H}+\frac{\eta}{2}\|f-f'\|_\mathcal{H}\leq R(f)-R(f'),
  %\end{aligned}
  \end{align}
  or $\forall  f,f'\in\mathcal{H}, t\in[0,1]$,
  \begin{align}
  \label{assumption-strongly-second}
  \begin{aligned}
  &~~~R(tf+(1-t)f')\\ &\leq  tR(f)+(1-t)R(f')-\frac{\eta t(1-t)}{2}\|f-f'\|_\mathcal{H}^2.
  \end{aligned}
  \end{align}

By \eqref{assumption-strongly-second}, one can see that
\begin{align*}
  R(\bar{f})&=R\left(\frac{1}{m}\sum_{i=1}^m\hat{f}_i\right)
  \\&\leq
  \frac{1}{m}\sum_{i=1}^mR(\hat{f}_i)-\frac{\eta}{4m^2}\sum_{i,j=1, i\not=j}^m\|\hat{f}_i-\hat{f}_j\|_\mathcal{H}^2\\
  %&=\frac{1}{m}\sum_{i=1}^mR(\hat{f}_i)-\eta\Delta_{\bar{f}}\\
  &\leq \frac{1}{m}\sum_{i=1}^mR(\hat{f}_i)-\eta\tau ~~(\text{by Assumption \ref{assumption-diversity}}).
\end{align*}

Therefore, we have
\begin{multline}
%\begin{aligned}
\label{equaiton-strongly-ff}
  R(\bar{f})-R(f_\ast)
  \leq \frac{1}{m}\sum_{i=1}^m\left[R(\hat{f}_i)-R(f_\ast)\right]-\eta\tau.
%\end{aligned}
\end{multline}

%Let $\hat{R}_i(f)=\frac{1}{|\mathcal{S}_i|}\sum_{z_j\in\mathcal{S}_i}\ell(f,z_j)$, $i=1,\ldots,m$.
As follows, we will estimate $R(\hat{f}_i)-R(f_\ast)$:
\begin{align}
\label{equaiton-ddd}
\begin{aligned}
 &~~~~R(\hat{f}_i)-R(f_\ast)+\frac{\eta}{2}\|\hat{f}_i-f_\ast\|_\mathcal{H}^2
 \\ &\overset{\eqref{assumption-strongly-equation}}{\leq} \langle \nabla R(\hat{f}_i),\hat{f}_i-f_\ast\rangle_\mathcal{H}
  \\&=
  \left\langle \nabla R(\hat{f}_i)-\nabla R(f_\ast)-[\nabla \hat{R}_i(\hat{f}_i)-\nabla \hat{R}_i(f_\ast)],
   \hat{f}_i-f_\ast\right\rangle_\mathcal{H}\\
   &~~~~+\left\langle\nabla R(f_\ast)-\nabla \hat{R}_i(f_\ast), \hat{f}_i-f_\ast\right\rangle_\mathcal{H}
   \\&~~~~+\left\langle \nabla \hat{R}_i(\hat{f}_i), \hat{f}_i-f_\ast\right\rangle_\mathcal{H}.
\end{aligned}
\end{align}
Note that $\ell(\cdot,z)$ is convex, thus $\hat{R}_i(\cdot)$ is convex.
By the convexity of $\hat{R}_i(\cdot)$ and the optimality condition of $\hat{f}_i$ \cite{boyd2004convex},
we have
\begin{align*}
%\label{equation-covenx}
    \left\langle \nabla \hat{R}_i(\hat{f}_i),f-\hat{f}_i\right\rangle_\mathcal{H}\geq 0, \forall f\in\mathcal{H}.
\end{align*}
Thus, we get
\begin{align}
\label{equation-covenx}
   \left \langle \nabla \hat{R}_i(\hat{f}_i),\hat{f}_i-f_\ast\right\rangle_\mathcal{H}\leq  0.
\end{align}
Substituting the above equation into \eqref{equaiton-ddd}, we have
\begin{align}
 \label{equation-important-middle}
 \begin{aligned}
  &~~~~~~R(\hat{f}_i)-R(f_\ast)+\frac{\eta}{2}\|\hat{f}_i-f_\ast\|_\mathcal{H}^2\\
  &\overset{\eqref{equation-covenx}}{\leq}\left\langle \nabla R(\hat{f}_i)-\nabla R(f_\ast)-\left[\nabla \hat{R}_i(\hat{f}_i)-\nabla \hat{R}_i(f_\ast)\right],
 \hat{f}_i-f_\ast\right\rangle_\mathcal{H}\\&~~~~+\left\langle\nabla R(f_\ast)-\nabla \hat{R}_i(f_\ast), \hat{f}_i-f_\ast\right\rangle_\mathcal{H}\\
  & \leq\left\|\nabla R(\hat{f}_i)-\nabla R(f_\ast)-\left[\nabla \hat{R}_i(\hat{f}_i)-\nabla \hat{R}_i(f_\ast)\right]\right\|_\mathcal{H}\cdot\left\|\hat{f}_i-f_\ast \right\|_\mathcal{H}
\\&~~~~+\left\|\nabla R(f_\ast)-\nabla \hat{R}_i(f_\ast)\right\|_\mathcal{H}\cdot\left\|\hat{f}_i-f_\ast \right\|_\mathcal{H}.
\end{aligned}
\end{align}
%From \eqref{equation-middle-00},
%we have
%\begin{align}
%  &\nonumber~~~~~R(\hat{f}_i)-R(f_\ast)+\frac{\eta}{2}\|\hat{f}_i-f_\ast\|_\mathcal{H}^2\\
%  &\nonumber\leq \langle \nabla R(\hat{f}_i)-\nabla R(f_\ast)-[\nabla \hat{R}_i(\hat{f}_i)-\nabla \hat{R}_i(f_\ast)],
%  \hat{f}_i-f_\ast\rangle_\mathcal{H}+\langle\nabla R(f_\ast)-\nabla \hat{R}_i(f_\ast), \hat{f}_i-f_\ast\rangle_\mathcal{H}\\
%   &\label{equation-important-middle}\leq \Big(\left\|\nabla R(\hat{f}_i)-\nabla R(f_\ast)-[\nabla \hat{R}_i(\hat{f}_i)-\nabla \hat{R}_i(f_\ast)]\right\|_\mathcal{H}
%+\left\|\nabla R(f_\ast)-\nabla \hat{R}_i(f_\ast)\right\|_\mathcal{H}\Big)\left\|\hat{f}_i-f_\ast \right\|_\mathcal{H}.
%\end{align}
As follows, we utilize the covering number to establish an upper bound for the first term in the last line of \eqref{equation-important-middle}.
The second term in the last line of \eqref{equation-important-middle} is upper bounded by the concentration inequality.

\subsection{Proof of Theorem \ref{theorem-log}}
To prove Theorem \ref{theorem-log},
we first introduce a lemma of \cite{smale2007learning},
and then provide two other lemmas.
\begin{lemma}[Lemma 2 of \cite{smale2007learning}]
\label{lem-one-first}
  Let $\mathcal{H}$ be a Hilbert space and $\xi$ be a random variable on ($\mathcal{Z},\rho$) with values in $\mathcal{H}$.
  Assume $$\|\xi\|_\mathcal{H}\leq \tilde{M}<\infty$$ almost surely.
  Denote $$\sigma^2(\xi)=\mathbb{E}(\|\xi\|_\mathcal{H}^2).$$
  Let $\{z_i\}_{i=1}^l$ be independent random drawers of $\rho$.
  For any $0<\delta<1$, with confidence $1-\delta$,
  \begin{align*}
    &\left\|\frac{1}{l}\sum_{i=1}^l\left[\xi_i-\mathbb{E}(\xi_i)\right]\right\|\\&~~~~~~~~~~~~~~~~\leq \frac{2\tilde{M}\log(2/\delta)}{l}+\sqrt{\frac{2\sigma^2(\xi)\log(2/\delta)}{l}}.
  \end{align*}
\end{lemma}

\begin{lemma}
\label{lemma-nablaR-nablahat}
If the loss function $\ell$ is a $G$-smooth and convex function,
 then for any $f\in\mathcal{N}(\mathcal{H},\epsilon)$, with a probability of at least $1-\delta$, we have
  \begin{align}
  \label{equation-nabalRR-empRR}
  \begin{aligned}
 % \nonumber
    &~~~~\left\|\nabla R(f)-\nabla R(f_\ast)-\left[\nabla \hat{R}_i(f)-\nabla \hat{R}_i(f_\ast)\right]\right\|_\mathcal{H} \\
   &\leq \frac{Gm\|f-f_\ast\|_\mathcal{H} D_{\mathcal{H},\delta,\epsilon}}{N}\\&~~~+\sqrt{\frac{Gm(R(f)-R(f_\ast)) D_{\mathcal{H},\delta,\epsilon}}{N}},
  \end{aligned}
  \end{align}
  where $D_{\mathcal{H},\delta,\epsilon}=2\log({2C(\mathcal{H},\epsilon)}/{\delta})$.
 % where $C(\epsilon,\delta)=2\left(\log\frac{2}{\delta}+|\mathcal{N}(\mathcal{H},\epsilon)|\right)$.
\end{lemma}
\begin{proof}
  Note that $\ell$ is $G$-smooth and convex, so by (2.1.7) of \cite{Nesterov2004},
  $\forall z\in\mathcal{Z}$,
  we have
  \begin{multline*}
  \left\|\nabla\ell (f,z)- \nabla\ell(f_\ast,z)\right\|_\mathcal{H}^2\\
   \leq
    G\left(\ell (f,z)-\ell (f_\ast,z)
    -\langle \nabla\ell (f_\ast,z), f-f_\ast\rangle_\mathcal{H} \right).
  \end{multline*}
  Taking the expectation over both sides, we have
  \begin{align}
  \begin{aligned}
  \label{eq-fagaga}
    &~~~\mathbb{E}_{z}\left[\left\|\nabla\ell (f,z)-\nabla\ell (f_\ast,z)\right\|_\mathcal{H}^2\right]\\
    &\leq G \left(R(f)-R(f_\ast)-\langle \nabla R(f_\ast), f-f_\ast\rangle_\mathcal{H}\right)\\
    &\leq G\Big(R(f)-R(f_\ast)\Big),
    %&\leq G\Big(R(f)-R(f_\ast)+\lambda\|f_\ast\|_\mathcal{H}^2-\lambda\|f\|_\mathcal{H}^2
%    \\&~~~~~~~~+\lambda\langle g_\ast,f-f_\ast\rangle_\mathcal{H}\Big)
  \end{aligned}
  \end{align}
  where the last inequality follows from the optimality condition of $f_\ast$, i.e.,
  $$\left\langle \nabla R(f_\ast),f-f_\ast\right\rangle_\mathcal{H} \geq 0,\forall f\in\mathcal{H}.$$

Note that $\ell(f,z)$ is $G$-smooth,
thus we have
\begin{align}
\label{eq-gasah}
  |\nabla \ell(f,z)-\nabla \ell(f_\ast,z)|\leq G\|f-f_\ast\|_\mathcal{H}, \forall f\in\mathcal{H}.
\end{align}

Substituting \eqref{eq-fagaga} and \eqref{eq-gasah} into Lemma \ref{lem-one-first} with $\xi_i=\nabla \ell(f,z_i)-\nabla \ell(f_\ast,z_i)$,
 we have
\begin{align*}
  &~~~\left\|
  \nabla R(f)-\nabla R(f_\ast)-[\nabla \hat{R}_i(f)-\nabla \hat{R}_i(f_\ast)]
  \right\|_\mathcal{H}\\
 % &=\left\|
%    \nabla R(f)-\nabla R(f_\ast)
%    -\frac{1}{n}\sum_{z_i\in \mathcal{S}_i}
%    \left[\nabla \ell(f,z_i)-\nabla \ell(f_\ast,z_i)\right]
%  \right\|_\mathcal{H}\\
  &=\left\|\frac{1}{n}\sum_{i=1}^{n}\left[\mathbb{E}(\xi_i)-\xi_i\right]\right\|&\\
  &\leq \frac{2mG\|f-f_\ast\|_\mathcal{H}\log({2}/{\delta})}{N}\\
  &~~~+\sqrt{\frac{2mG(R(f)-R(f_\ast))\log({2}/{\delta})}{N}}.
\end{align*}

We obtain Lemma \ref{lemma-nablaR-nablahat} by taking the union bound over all $f\in\mathcal{N}(\mathcal{H},\epsilon)$.
\end{proof}

\begin{lemma}
\label{lemma-second-ff}
Under Assumption \ref{assumption-G-L},
 with a probability of at least $1-\delta$, we have
  \begin{multline}
  \label{equation-nablaR-nablaempR}
    \left\|\nabla R(f_\ast)-\nabla \hat{R}_i(f_\ast)\right\|_\mathcal{H}\\
      \leq \frac{2Lm\log(\frac{2}{\delta})}{N}+\sqrt{\frac{8GR_\ast m\log(\frac{2}{\delta})}{N}}.
    \end{multline}
\end{lemma}
\begin{proof}
  Since $\ell(f,\cdot)$ is $G$-smooth and non-negative,
  from Lemma 4 of \cite{srebro2010optimistic}, we have
  $$
    \left\|\nabla \ell(f_\ast,z_i)\right\|_\mathcal{H}^2\leq 4G \ell(f_\ast,z_i)
  $$
  and thus we can get
    \begin{align*}
      &\mathbb{E}_{z}\left[\left\|\nabla \ell(f_\ast,z)\right\|_\mathcal{H}^2\right]\\ &~~~~~~~~\leq 4G\mathbb{E}_{z}[\ell(f_\ast,z)]=
      4G R(f_\ast).
    \end{align*}
    Since $\ell(f,\cdot)$ is a $L$-Lipschitz continuous function,
    we have
    \begin{align*}
     &\|\ell(f_\ast+\delta_f,z)-\ell(f_\ast,z)\|_\mathcal{H} \\ &~~~~~~~~~~~~~~~\leq L\|\delta_f\|_\mathcal{H}, \forall \delta_f\in\mathcal{H}.
    \end{align*}
    If $\|\delta_f\|_\mathcal{H}\rightarrow 0$, from the definition of differential of $\ell(f_\ast,z)$,
    we can obtain that
    \begin{align*}
      \|\nabla \ell(f_\ast,z)\|_\mathcal{H}\leq L.
    \end{align*}
   % $$\|\nabla \ell(f_\ast,z)\|_\mathcal{H}\leq L.$$
   % Let $H(f)=R(f)-r(f)$ and $\hat{H}(f)=\hat{R}(f)-r(f)$.
    Then, according to Lemma \ref{lem-one-first} with $\xi_i=\nabla\ell(f_\ast,z_i)$,  we have
    \begin{multline*}
     \left\|\nabla R(f_\ast)-\nabla \hat{R}_i(f_\ast)\right\|_\mathcal{H}
      \\\leq  \frac{2Lm\log({2}/{\delta})}{N}+\sqrt{\frac{8GR_\ast m\log({2}/{\delta})}{N}}.
    \end{multline*}
\end{proof}

\begin{proof}[\textbf{Proof of Theorem \ref{theorem-log}}]
From the properties of $\epsilon$-covering, we know that there exists a function
$\tilde{f}\in\mathcal{N}(\mathcal{H},\epsilon)$ such that
\begin{align}
  \label{epislon-fa}
  \|\hat{f}_i-\tilde{f}\|_\mathcal{H}\leq \epsilon.
\end{align}
%Since $\ell$ is $G$-smooth,
%so $R(f)$ and $\hat{R}_i(f)$ are both $G$-smooth.
Thus, we have
\begin{align}
\begin{aligned}
\label{equation-31}
    &~~~\left\|\nabla R(\hat{f}_i)-\nabla R(f_\ast)-\left[\nabla \hat{R}_i(\hat{f}_i)-\nabla \hat{R}_i(f_\ast)\right]\right\|_\mathcal{H} \\
    &\leq \left\|\nabla R(\tilde{f})-\nabla R(f_\ast)-\left[\nabla \hat{R}_i(\tilde{f})-\nabla \hat{R}_i(f_\ast)\right]\right\|_\mathcal{H}\\
    &~~~+\left\|\nabla R(\hat{f}_i)-\nabla R(\tilde{f})\right\|_\mathcal{H}+\left\| \nabla \hat{R}_i(\hat{f}_i)-\nabla \hat{R}_i(\tilde{f}) \right\|_\mathcal{H}\\
    &\leq \left\|\nabla R(\tilde{f})-\nabla R(f_\ast)-\left[\nabla \hat{R}_i(\tilde{f})-\nabla \hat{R}_i(f_\ast)\right]\right\|_\mathcal{H}\\&~~~+2G\|\hat{f}_i-\tilde{f}\|_\mathcal{H}~(\text{by $G$-smooth})\\
    &\overset{\eqref{epislon-fa}}{\leq} \left\|\nabla R(\tilde{f})-\nabla R(f_\ast)-\left[\nabla \hat{R}_i(\tilde{f})-\nabla \hat{R}_i(f_\ast)\right]\right\|_\mathcal{H}\\&~~~~~+2G \epsilon\\
    &~\overset{\eqref{equation-nabalRR-empRR}}{\leq}
   \frac{G  D_{\mathcal{H},\delta,\epsilon}\|\tilde{f}-f_\ast\|_\mathcal{H}m}{N}\\&~~~~~+
   \sqrt{\frac{G  D_{\mathcal{H},\delta,\epsilon}(R(\tilde{f})-R(f_\ast))m}{N}}+2G\epsilon\\
   &\leq \frac{G  D_{\mathcal{H},\delta,\epsilon}\|\hat{f}_i-f_\ast\|_\mathcal{H}m}{N}\\&~~~+\frac{G  D_{\mathcal{H},\delta,\epsilon}\|\hat{f}_i-\tilde{f}\|_\mathcal{H}m}{N}\\
   &~~~+\sqrt{\frac{G  D_{\mathcal{H},\delta,\epsilon}(R(\hat{f}_i)-R(f_\ast))m}{N}}
   \\&~~~+\sqrt{\frac{G  D_{\mathcal{H},\delta,\epsilon}|R(\hat{f}_i)-R(\tilde{f})|m}{N}}+2G\epsilon\\
    &\leq \frac{G  D_{\mathcal{H},\delta,\epsilon}\|\hat{f}_i-f_\ast\|_\mathcal{H}m}{N}\\&~~~~~+\frac{G \epsilon m D_{\mathcal{H},\delta,\epsilon}}{N} ~(\text{by \eqref{epislon-fa}})\\
   &~~~~~+\sqrt{\frac{G  D_{\mathcal{H},\delta,\epsilon}(R(\hat{f}_i)-R(f_\ast))m}{N}}\\&~~~~~+
   \sqrt{\frac{GL \epsilon m D_{\mathcal{H},\delta,\epsilon}}{N}}+2G\epsilon ~(\text{by $L$-Lipschitz}).
\end{aligned}
\end{align}
%where the last inequality holds by $\left|R(\hat{f}_i)-R(\tilde{f})\right|\leq \|\hat{f}_i-\tilde{f}\|_\mathcal{H}
%+\lambda(\|\hat{f}_i\|_\mathcal{H}+\|\tilde{f}_i\|_\mathcal{H})\|\hat{f}_i-\tilde{f}\|_\mathcal{H}\leq (G+2M)\epsilon.$
  Substituting \eqref{equation-31} and \eqref{equation-nablaR-nablaempR} into \eqref{equation-important-middle},
  with a probability of at least $1-2\delta$, we have
  \begin{align}
    \label{equation-first-equation}
    \begin{aligned}
      &~~~~~R(\hat{f}_i)-R(f_\ast)+\frac{\eta}{2}\|\hat{f}_i-f_\ast\|_\mathcal{H}^2\\
      &\leq \frac{G  D_{\mathcal{H},\delta,\epsilon}\|\hat{f}_i-f_\ast\|_\mathcal{H}^2m}{N}\\&~~~+
   \frac{G  D_{\mathcal{H},\delta,\epsilon}\epsilon \|\hat{f}_i-f_\ast\|_\mathcal{H}m}{N}
   \\&~~~+2G\epsilon \|\hat{f}_i-f_\ast\|_\mathcal{H}\\&~~~+\|\hat{f}_i-f_\ast\|_\mathcal{H}\sqrt{\frac{G  D_{\mathcal{H},\delta,\epsilon}(R(\hat{f}_i)-R(f_\ast))m}{N}}
   \\&~~~+\|\hat{f}_i-f_\ast\|_\mathcal{H}\sqrt{\frac{GL \epsilon m D_{\mathcal{H},\delta,\epsilon}}{N}}\\&~~~
 +\frac{2L \log(\frac{2}{\delta})\|\hat{f}_i-f_\ast\|_\mathcal{H}m}{N}
   \\&~~~+\|\hat{f}_i-f_\ast\|_\mathcal{H}\sqrt{\frac{8G R_\ast m \log(\frac{2}{\delta})}{N}}.
    \end{aligned}
  \end{align}

  Note that
  $$
    \sqrt{ab}\leq \frac{a}{2c}+\frac{bc}{2}, \forall a,b,c> 0.
  $$
  Therefore, we have
  \begin{align*}
  %\label{equation-finally}
%  \begin{aligned}
    &~~~\|\hat{f}_i-f_\ast\|_\mathcal{H}\sqrt{\frac{G \log D_{\mathcal{H},\delta,\epsilon}(R(\hat{f}_i)-R(f_\ast))m}{N}}
    \\&\leq \frac{2G  D_{\mathcal{H},\delta,\epsilon}(R(\hat{f}_i)-R(f_\ast))m}{N\eta}+\frac{\eta}{8}\|\hat{f}_i-f_\ast\|_\mathcal{H}^2;\\
    &~~~\frac{2L\log(\frac{2}{\delta})\|\hat{f}_i-f_\ast\|_\mathcal{H}m}{N}
    \\&\leq \frac{16L^2m^2\log^2\frac{2}{\delta}}{N^2\eta}+\frac{\eta}{16}\|\hat{f}_i-f_\ast\|_\mathcal{H}^2;\\
    &~~~\|\hat{f}_i-f_\ast\|_\mathcal{H}\sqrt{\frac{8GR_\ast \log(\frac{2}{\delta})m}{N}}
    \\&\leq \frac{64G R_\ast\log(\frac{2}{\delta})m}{N\eta}+\frac{\eta }{32}\|\hat{f}_i-f_\ast\|_\mathcal{H}^2;\\
    &~~~2G\epsilon \|\hat{f}_i-f_\ast\|_\mathcal{H}
    \\&\leq \frac{64G^2\epsilon^2}{\eta}+\frac{\eta }{64}\|\hat{f}_i-f_\ast\|_\mathcal{H}^2;\\
    &~~~\|\hat{f}_i-f_\ast\|_\mathcal{H}\sqrt{\frac{GL \epsilon m D_{\mathcal{H},\delta,\epsilon}}{N}}
    \\&\leq \frac{32GL \epsilon m D_{\mathcal{H},\delta,\epsilon}}{N\eta}+\frac{\eta}{128}\|\hat{f}_i-f_\ast\|_\mathcal{H}^2;\\
    &~~~\frac{G  D_{\mathcal{H},\delta,\epsilon}\epsilon m \|\hat{f}_i-f_\ast\|_\mathcal{H}}{N}
    \\&\leq \frac{32G \epsilon^2m^2 D^2_{\mathcal{H},\delta,\epsilon}}{N^2\eta}+\frac{\eta }{128}\|\hat{f}_i-f_\ast\|_\mathcal{H}^2.
  %\end{aligned}
  \end{align*}
 Substituting the above  inequalities into \eqref{equation-first-equation}, we have
 \begin{align}
 \begin{aligned}
 \label{eq-middel-new}
   &~~~~R(\hat{f}_i)-R(f_\ast)+\frac{\eta}{4}\|\hat{f}_i-f_\ast\|_\mathcal{H}^2\\
   &\leq \frac{G  D_{\mathcal{H},\delta,\epsilon}\|\hat{f}_i-f_\ast\|_\mathcal{H}^2m}{N}\\
   &~~~+\frac{2G  D_{\mathcal{H},\delta,\epsilon}(R(\hat{f}_i)-R(f_\ast))m}{N\eta}\\&~~~+\frac{16L^2m^2\log^2(\frac{2}{\delta})}{N^2\eta}\\
   &~~~+\frac{64G R_\ast m\log(\frac{2}{\delta})}{N\eta}\\&~~~+\frac{64G^2\epsilon^2}{\eta}+
   \frac{32GL \epsilon m D_{\mathcal{H},\delta,\epsilon}}{N\eta}
   \\&~~~+\frac{32G \epsilon^2 m^2 D^2_{\mathcal{H},\delta,\epsilon}}{N^2\eta}.
  % &\overset{\eqref{equation-12}}{\leq}
%   \frac{\eta}{4}\|\hat{f}_i-f_\ast\|_\mathcal{H}^2+\frac{1}{2}(R(\hat{f}_i)-R(f_\ast))+\frac{16L \log(\frac{2}{\delta})}{n^2\eta}+\frac{64G R_\ast\log(\frac{2}{\delta})}{n\eta}\\&~~~+\frac{64G^2\epsilon^2}{\eta}+
%   \frac{32GL  D_{\mathcal{H},\delta,\epsilon}\epsilon}{n\eta}+\frac{32G D^2(\mathcal{H},\delta,\epsilon)\epsilon^2}{n^2\eta}.
 \end{aligned}
 \end{align}
  From Assumption \ref{log-metric}, we know that
  $
    \log C(\mathcal{H},\epsilon)\simeq h\log(1/\epsilon).
  $
  Thus, we can obtain that
  \begin{align}
  \label{eq-midm}
  \begin{aligned}
    D_{\mathcal{H},\delta,\epsilon}
   % &=2\left(\log C(\mathcal{H},\epsilon)+\log\frac{2}{\delta}\right)\\&
    =2h\log \left(\frac{2}{\delta\epsilon}\right).
  \end{aligned}
  \end{align}
 If we set $\epsilon=\frac{1}{N}$,
 substituting \eqref{eq-midm} into \eqref{eq-middel-new},
 we have
  \begin{align*}
   &~~~~R(\hat{f}_i)-R(f_\ast)+\frac{\eta}{4}\|\hat{f}_i-f_\ast\|_\mathcal{H}^2\\
   &\leq \frac{2Gh\log(2N/\delta)\|\hat{f}_i-f_\ast\|_\mathcal{H}^2m}{N}\\
   &~~~+\frac{4Gh\log(2N/\delta)(R(\hat{f}_i)-R(f_\ast))m}{N\eta}\\&~~~+\frac{16L^2m^2\log^2(\frac{2}{\delta})}{N^2\eta}
   +\frac{64G R_\ast m\log(\frac{2}{\delta})}{N\eta}\\&~~~+\frac{64G^2}{N^2}+
   \frac{64GL m h\log(2N/\delta)}{N^2\eta}
   \\&~~~+\frac{128G m^2 h^2\log^2(2N/\delta)}{N^4\eta}.
  % &\overset{\eqref{equation-12}}{\leq}
%   \frac{\eta}{4}\|\hat{f}_i-f_\ast\|_\mathcal{H}^2+\frac{1}{2}(R(\hat{f}_i)-R(f_\ast))+\frac{16L \log(\frac{2}{\delta})}{n^2\eta}+\frac{64G R_\ast\log(\frac{2}{\delta})}{n\eta}\\&~~~+\frac{64G^2\epsilon^2}{\eta}+
%   \frac{32GL  D_{\mathcal{H},\delta,\epsilon}\epsilon}{n\eta}+\frac{32G D^2(\mathcal{H},\delta,\epsilon)\epsilon^2}{n^2\eta}.
 \end{align*}
 Thus, when
  $m\leq \frac{N\eta}{8G h\log(2N/\delta)},$
one can obtain that
 \begin{align*}
   &~~~R(\hat{f}_i)-R(f_\ast)+\frac{\eta}{4}\|\hat{f}_i-f_\ast\|_\mathcal{H}^2
   \\&\leq \frac{\eta}{4}\|\hat{f}_i-f_\ast\|_\mathcal{H}^2+\frac{1}{2}(R(\hat{f}_i)-R(f_\ast))\\&~~~+\frac{16L^2 \log^2(\frac{2}{\delta})m^2}{N^2\eta}+\frac{64G R_\ast m\log(\frac{2}{\delta})}{N\eta}+
   \frac{64G^2}{N^2\eta}\\&~~~+
   \frac{64GLh\log(2N/\delta)m}{N^2\eta}+\frac{128Gh\log^2(2N/\delta) m^2}{N^4\eta}.
 \end{align*}
 Thus, we have
 \begin{align}
  \label{theorem-middle}
  \begin{aligned}
    &~~~R(\hat{f}_i)-R(f_\ast)
    \\&\leq\frac{32L^2 m^2\log^2(\frac{2}{\delta})}{N^2\eta}+\frac{128G R_\ast m\log(\frac{2}{\delta})}{N\eta}\\&~~~+\frac{128G^2}{N^2\eta}
   + \frac{128GLh\log(2N/\delta)m}{N^2\eta}\\&~~~+\frac{256Gh\log^2(2N/\delta) m^2}{N^4\eta}.
   \end{aligned}
  \end{align}
Substituting \eqref{theorem-middle} into \eqref{equaiton-strongly-ff},
we have
 \begin{align}
  \label{eq-proof-theorem221}
  \begin{aligned}
    &~~~R(\bar{f})-R(f_\ast)
    \\&\leq\frac{32L^2 m^2\log^2(\frac{2}{\delta})}{N^2\eta}+\frac{128G R_\ast m\log(\frac{2}{\delta})}{N\eta}\\&~~~+\frac{128G^2}{N^2\eta}
    +\frac{128GLh\log(2N/\delta)m}{N^2\eta}\\&~~~+\frac{256Gh\log^2(2N/\delta) m^2}{N^4\eta}-\eta\tau.
   \end{aligned}
  \end{align}

 %% Note that $D_{\mathcal{H},\delta,\epsilon}=2\left(C(\mathcal{H},\epsilon)+\log\frac{2}{\delta}\right)$,
%  From Assumption \ref{log-metric}, we know that
%  $$
%    \log C(\mathcal{H},\epsilon)\simeq h\log(1/\epsilon).
%  $$
%  Thus, we can obtain that
%  \begin{align}
%  \label{eq-midm}
%    D_{\mathcal{H},\delta,\epsilon}=2\Big(\log C(\mathcal{H},\epsilon)+\log{2}/{\delta}\Big)
%    =2h\log ({2}/{\delta\epsilon}).
%  \end{align}
%  If we set $\epsilon=1/N$,
%  from \eqref{theorem-middle-newbound} and \eqref{eq-midm},
%  we can obtain that
%  \begin{align}
%  \begin{aligned}
%  \label{eq-proof-theorem221}
%    &~~~R(\bar{f})-R(f_\ast)\\
%%    &\leq \Omega\left(\frac{m^2\log^2\frac{1}{\delta}}{N^2}+\frac{R_\ast m\log\frac{m}{\delta}}{N}+\frac{m^2D^2(\mathcal{H,\delta,\epsilon})}{N^4}+
%%    \frac{mD(\mathcal{H,\delta,\epsilon})}{N^2}-\eta\Delta_{\bar{f}} \right) \\
%    &\leq O\left(\frac{m^2\log^2\frac{1}{\delta}}{N^2}+\frac{R_\ast m\log\frac{1}{\delta}}{N}+\frac{m^2h^2\log^2\frac{N}{\delta}}{N^4}\right.\\&~~~~~~~~~+
%    \left.\frac{mh\log\frac{N}{\delta}}{N^2}-\eta\Delta_{\bar{f}}\right).
%  \end{aligned}
%  \end{align}
  Note that when
  \begin{align*}
    m\leq \min\left\{\frac{h\eta+N\eta^2\tau}{128GR_\ast\log(2/\delta)}, \frac{\sqrt{Nh\eta}}{L\log(2/\delta)},\frac{N\eta}{GL\log(2N/\delta)}\right\},
  \end{align*}
one can obtain
  \begin{align*}
   \frac{128 G R_\ast m\log\frac{2}{\delta}}{N\eta}-\eta\tau&\leq \frac{h}{N},\\
    \frac{L^2m^2\log^2\frac{2}{\delta}}{N^2\eta}&\leq \frac{h}{N},\\
    \frac{GLh\log(2N/\delta)m}{N^2\eta}&\leq \frac{h}{N}.
  \end{align*}
  Therefore, substituting the above equations into \eqref{eq-proof-theorem221},
  we have
  \begin{align*}
   R(\bar{f})-R(f_\ast)
   &\leq O\left(\frac{h}{N}+\frac{1}{N^2}+\frac{h m^2 \log^2(N)}{N^4}\right)\\
   &=O\left(\frac{h}{N}\right).
  \end{align*}
\end{proof}

\subsection{Proof of Theorem \ref{theorem-poly}}
\begin{proof}
  According to Assumption \ref{poly-metric}, we know that
  $$
    \log C(\mathcal{H},\epsilon)\simeq (1/\epsilon)^{1/h}.
  $$
  Thus, when setting $\epsilon=N^{-\frac{h}{2h+1}}$, one can see that
  \begin{align}
  \begin{aligned}
  \label{eq-midmdsf}
    D_{\mathcal{H},\delta,\epsilon}&=2\left(\log C(\mathcal{H},\epsilon)+\log\frac{2}{\delta}\right)
    \\&=2\left(N^\frac{1}{2h+1}+\log\frac{2}{\delta}\right).
  \end{aligned}
  \end{align}
  From \eqref{eq-middel-new} with $\epsilon=N^{-\frac{h}{2h+1}}$,
  we have
   \begin{align}
 \begin{aligned}
 \label{eq-middel-new-ploy}
   &~~~~R(\hat{f}_i)-R(f_\ast)+\frac{\eta}{4}\|\hat{f}_i-f_\ast\|_\mathcal{H}^2\\
   &\leq \frac{G  D_{\mathcal{H},\delta,\epsilon}\|\hat{f}_i-f_\ast\|_\mathcal{H}^2m}{N}\\
   &~~~+\frac{2G  D_{\mathcal{H},\delta,\epsilon}(R(\hat{f}_i)-R(f_\ast))m}{N\eta}\\&~~~+\frac{16L^2m^2\log^2(\frac{2}{\delta})}{N^2\eta}\\
   &~~~+\frac{64G R_\ast m\log(\frac{2}{\delta})}{N\eta}\\&~~~+\frac{64G^2}{\eta N^{\frac{2h}{2h+1}}}+
   \frac{32GL m D_{\mathcal{H},\delta,\epsilon}}{\eta N^{\frac{3h+1}{2h+1}}}
   \\&~~~+\frac{32G m^2 D^2_{\mathcal{H},\delta,\epsilon}}{\eta N^{2+\frac{2h}{2h+1}}}.
  % &\overset{\eqref{equation-12}}{\leq}
%   \frac{\eta}{4}\|\hat{f}_i-f_\ast\|_\mathcal{H}^2+\frac{1}{2}(R(\hat{f}_i)-R(f_\ast))+\frac{16L \log(\frac{2}{\delta})}{n^2\eta}+\frac{64G R_\ast\log(\frac{2}{\delta})}{n\eta}\\&~~~+\frac{64G^2\epsilon^2}{\eta}+
%   \frac{32GL  D_{\mathcal{H},\delta,\epsilon}\epsilon}{n\eta}+\frac{32G D^2(\mathcal{H},\delta,\epsilon)\epsilon^2}{n^2\eta}.
 \end{aligned}
 \end{align}
  Thus, when
  \begin{align*}
    m\leq \frac{N\eta}{4D_{\mathcal{H},\delta,\epsilon}}=\frac{N\eta}{4\big(N^{\frac{1}{2h+1}}+\log (2/\delta)\big)},
  \end{align*}
   we have
   \begin{align*}
 %\begin{aligned}
% \label{eq-middel-new-ploy}
   &~~~~R(\hat{f}_i)-R(f_\ast)+\frac{\eta}{4}\|\hat{f}_i-f_\ast\|_\mathcal{H}^2\\
   &\leq \frac{\eta}{4}\|\hat{f}_i-f_\ast\|_\mathcal{H}^2+\frac{1}{2}\left(R(\hat{f}_i)-R(f_\ast)\right)
   \\&~~~+\frac{16L^2m^2\log^2(\frac{2}{\delta})}{N^2\eta}\\
   &~~~+\frac{64G R_\ast m\log(\frac{2}{\delta})}{N\eta}\\&~~~+\frac{64G^2}{N^{\frac{2h}{2h+1}}\eta}+
   \frac{32GL m D_{\mathcal{H},\delta,\epsilon}}{N^{\frac{3h+1}{2h+1}}\eta}
   \\&~~~+\frac{32G m^2 D^2_{\mathcal{H},\delta,\epsilon}}{N^{2+\frac{2h}{2h+1}}\eta}.
  % &\overset{\eqref{equation-12}}{\leq}
%   \frac{\eta}{4}\|\hat{f}_i-f_\ast\|_\mathcal{H}^2+\frac{1}{2}(R(\hat{f}_i)-R(f_\ast))+\frac{16L \log(\frac{2}{\delta})}{n^2\eta}+\frac{64G R_\ast\log(\frac{2}{\delta})}{n\eta}\\&~~~+\frac{64G^2\epsilon^2}{\eta}+
%   \frac{32GL  D_{\mathcal{H},\delta,\epsilon}\epsilon}{n\eta}+\frac{32G D^2(\mathcal{H},\delta,\epsilon)\epsilon^2}{n^2\eta}.
 %\end{aligned}
 \end{align*}
 Thus, one can obtain that
 \begin{align*}
    &~~~~R(\hat{f}_i)-R(f_\ast)\\
   \\&\leq\frac{32L^2m^2\log^2(\frac{2}{\delta})}{N^2\eta}+\frac{128G R_\ast m\log(\frac{2}{\delta})}{N\eta}\\&~~~+\frac{128G^2}{N^{\frac{2h}{2h+1}}\eta}+
   \frac{64GL m D_{\mathcal{H},\delta,\epsilon}}{N^{\frac{3h+1}{2h+1}}\eta}\\&~~~+\frac{64G m^2 D^2_{\mathcal{H},\delta,\epsilon}}{N^{2+\frac{2h}{2h+1}}\eta}.
 \end{align*}
  Substituting the above inequality into \eqref{equaiton-strongly-ff},
we have
  \begin{align}
  \begin{aligned}
  \label{eq-proof-34353}
     &~~~~R(\bar{f})-R(f_\ast)\\
   \\&\leq\frac{32L^2m^2\log^2(\frac{2}{\delta})}{N^2\eta}\\&~~~+\frac{128G R_\ast m\log(\frac{2}{\delta})}{N\eta}\\&~~~+\frac{128G^2}{N^{\frac{2h}{2h+1}}\eta}+
   \frac{64GL m D_{\mathcal{H},\delta,\epsilon}}{N^{\frac{3h+1}{2h+1}}\eta}
   \\&~~~+\frac{64G m^2 D^2_{\mathcal{H},\delta,\epsilon}}{N^{2+\frac{2h}{2h+1}}\eta}-\eta\tau.
 \end{aligned}
 \end{align}
 %  \begin{align}
%  \begin{aligned}
%  \label{eq-proof-34353}
%    &~~~~~R(\bar{f})-R(f_\ast)\\
%    &\leq O\Big(\frac{m^2\log^2\frac{1}{\delta}}{N^2}+\frac{R_\ast m\log\frac{1}{\delta}}{N}+\frac{m^2D^2(\mathcal{H,\delta,\epsilon})}{N^4}
%    \\&~~~~~~~~+\frac{mD(\mathcal{H,\delta,\epsilon})}{N^2}-\eta\Delta_{\bar{f}} \Big)\\
%    &\overset{\eqref{eq-midmdsf}}{=}O\Big(\frac{m^2\log^2\frac{1}{\delta}}{N^2}+\frac{R_\ast m\log\frac{1}{\delta}}{N}+\frac{m^2N^\frac{2}{h}
%    }{N^4}\\&~~~~~~~+\frac{m^2\log\frac{2}{\delta}}{N^4}+\frac{m}{N^{2-\frac{1}{h}}}+\frac{m\log\frac{1}{\delta}}{N^2}-\eta\Delta_{\bar{f}}\Big).
%  \end{aligned}
%  \end{align}
Note that,
$$D_{\mathcal{H},\delta,\epsilon}=2\left(N^{\frac{1}{2h+1}}+\log\frac{2}{\delta}\right)\leq 2N^{\frac{1}{2h+1}}\log\frac{2}{\delta}.$$
%thus $D_{\mathcal{H},\delta,\epsilon}=4\left(N^{\frac{2}{h}}+2N^{\frac{1}{h}}\log\frac{2}{\delta}+\log^2\frac{2}{\delta}\right)$.
Thus, from \eqref{eq-proof-34353}, we can obtain that
\begin{align}
\begin{aligned}
\label{eq-fagagkjhga}
 &~~~~R(\bar{f})-R(f_\ast)\\
   \\&\leq\frac{32L^2m^2\log^2(\frac{2}{\delta})}{N^2\eta}+\frac{128G R_\ast m\log(\frac{2}{\delta})}{N\eta}
   \\&~~~+\frac{128G^2}{N^{\frac{2h}{2h+1}}\eta}+\frac{128 GLm\log\frac{2}{\delta}}{N^{\frac{3h}{2h+1}}\eta}+
   \\&~~~+\frac{256G m^2\log^2(2/\delta)}{N^{\frac{6h}{2h+1}}\eta}-\eta\tau.
\end{aligned}
\end{align}

  Note that, when
  $$m\leq \min\left\{\frac{\sqrt{\eta}N^{\frac{h+1}{2h+1}}}{L\log(2/\delta)},\frac{\eta N^{\frac{h}{2h+1}}}{GL\log(2/\delta)},
  \frac{N^{\frac{1}{2h+1}}\eta+N\eta^2\tau}{128 R_\ast \log(2/\delta)}
  \right\},$$
  one can obtain that
  \begin{align}
  \begin{aligned}
  \label{eq-fagjagl}
    \frac{L^2m^2\log^2(2/\delta)}{N^2\eta}&\leq N^{-\frac{2h}{2h+1}},\\
    \frac{GLm\log\frac{2}{\delta}}{N^{\frac{3h}{2h+1}}\eta}&\leq N^{-\frac{2h}{2h+1}},\\
    \frac{128 GR_\ast m\log\frac{2}{\delta}}{N\eta}-\eta\tau&\leq N^{-\frac{2h}{2h+1}}.
  \end{aligned}
  \end{align}
Substituting \eqref{eq-fagjagl} into \eqref{eq-fagagkjhga},
  we have
  \begin{align*}
    R(\bar{f})-R(f_\ast)\leq O\left(N^{-\frac{2h}{2h+1}}
    +\frac{m^2\log^2\frac{2}{\delta}}{N^{\frac{6h}{2h+1}}}\right).
    %&=\tilde{O}\left(N^{-\frac{2h}{2h+1}}+\frac{m^2N^{\frac{2}{h}}}{N^4}\right).
  \end{align*}
  By \eqref{eq-fagjagl},
  we know that $$O\left(\frac{m\log\frac{2}{\delta}}{N^{\frac{3h}{2h+1}}}\right)\leq O\left(N^{-\frac{2h}{2h+1}}\right).$$
  Thus, we have
  \begin{align*}
    R(\bar{f})-R(f_\ast)&\leq O\left(N^{-\frac{2h}{2h+1}}+N^{-\frac{4h}{2h+1}}\right)\\
    &= O\left(N^{-\frac{2h}{2h+1}}\right).
  \end{align*}
\end{proof}
\subsection{Proof of Theorem \ref{the-non-strong}}
\begin{proof}
We set $\eta=0$ in \eqref{equation-first-equation},
and obtain that
  \begin{align}
    \label{equation-first-equation-nons}
    \begin{aligned}
      &~~~R(\hat{f}_i)-R(f_\ast)\\&
      \leq  \frac{G  D_{\mathcal{H},\delta,\epsilon}\|\hat{f}_i-f_\ast\|_\mathcal{H}^2m}{N}+
   \frac{G  D_{\mathcal{H},\delta,\epsilon}\epsilon \|\hat{f}_i-f_\ast\|_\mathcal{H}m}{N}
   \\&~~~+2G\epsilon \|\hat{f}_i-f_\ast\|_\mathcal{H}
   \\&~~~+\|\hat{f}_i-f_\ast\|_\mathcal{H}\sqrt{\frac{G  D_{\mathcal{H},\delta,\epsilon}(R(\hat{f}_i)-R(f_\ast))m}{N}}\\&~~~+
   \|\hat{f}_i-f_\ast\|_\mathcal{H}\sqrt{\frac{GL \epsilon m D_{\mathcal{H},\delta,\epsilon}}{N}} \\&~~~+\frac{2L \log(\frac{2}{\delta})\|\hat{f}_i-f_\ast\|_\mathcal{H}m}{N}\\&~~~
   +\|\hat{f}_i-f_\ast\|_\mathcal{H}\sqrt{\frac{8G R_\ast m \log(\frac{2}{\delta})}{N}}.
    \end{aligned}
  \end{align}
  Note that
  $$
    \sqrt{ab}\leq \frac{a}{2c}+\frac{bc}{2}, \forall a,b,c\geq 0.
  $$
  Thus, we have
  \begin{align}
    \label{mid-nonstrong}
    \begin{aligned}
    &~~~\|\hat{f}_i-f_\ast\|_\mathcal{H}\sqrt{\frac{G mD_{\mathcal{H},\delta,\epsilon}(R(\hat{f}_i)-R(f_\ast))}{N}}\\&\leq
    \frac{G m D_{\mathcal{H},\delta,\epsilon}\|\hat{f}_i-f_\ast\|_\mathcal{H}^2}{2N}+\frac{R(\hat{f}_i)-R(f_\ast)}{2};\\
    &~~~\|\hat{f}_i-f_\ast\|_\mathcal{H}\sqrt{\frac{GL m \epsilon D_{\mathcal{H},\delta,\epsilon}}{N}}
    \\&\leq \frac{L mD_{\mathcal{H},\delta,\epsilon} \|\hat{f}_i-f_\ast\|^2_\mathcal{H}}{2N}+\frac{G\epsilon}{2}.
  \end{aligned}
  \end{align}
  Substituting \eqref{mid-nonstrong} into \eqref{equation-first-equation-nons},
  we get
  \begin{align}
  \label{eq-mide-ddfg}
  \begin{aligned}
    &~~~\frac{1}{2}(R(\hat{f}_i)-R(f_\ast))
    \\&\leq
    \frac{G m D_{\mathcal{H},\delta,\epsilon}\|\hat{f}_i-f_\ast\|_\mathcal{H}^2}{N}+
   \frac{G m \epsilon D_{\mathcal{H},\delta,\epsilon} \|\hat{f}_i-f_\ast\|_\mathcal{H}}{N}
   \\&~~~+2G\epsilon \|\hat{f}_i-f_\ast\|_\mathcal{H}+\frac{G m D_{\mathcal{H},\delta,\epsilon}\|\hat{f}_i-f_\ast\|_\mathcal{H}^2}{2N}
   \\&~~~+\frac{L m D_{\mathcal{H},\delta,\epsilon} \|\hat{f}_i-f_\ast\|^2_\mathcal{H}}{2N}+\frac{G\epsilon}{2}\\
   &~~~ +\frac{2L m \log(\frac{2}{\delta})\|\hat{f}_i-f_\ast\|_\mathcal{H}}{N}
   \\&~~~+\|\hat{f}_i-f_\ast\|_\mathcal{H}\sqrt{\frac{8G m R_\ast \log(\frac{2}{\delta})}{N}}.
  \end{aligned}
  \end{align}
  If $\forall f\in \mathcal{H}$, $\|f\|_\mathcal{H}\leq B$, one can see that
  $\|\hat{f}_i-f_\ast\|_\mathcal{H}\leq 2B$.
  Thus, from \eqref{eq-mide-ddfg},
  we have
  \begin{align}
  \begin{aligned}
  \label{the-eq-minggd}
    &~~~R(\hat{f}_i)-R(f_\ast)
    \\&\leq
    \frac{8G B^2 mD_{\mathcal{H},\delta,\epsilon}}{N}+
   \frac{4BG m D_{\mathcal{H},\delta,\epsilon}\epsilon}{N}+8BG\epsilon\\&~~~+\frac{4B^2G m D_{\mathcal{H},\delta,\epsilon}}{N}
   +\frac{4B^2L mD_{\mathcal{H},\delta,\epsilon}}{N}+G\epsilon\\&~~~+\frac{8BLm \log(\frac{2}{\delta})}{N}+
   4B\sqrt{\frac{8G m R_\ast \log(\frac{2}{\delta})}{N}}.
  \end{aligned}
  \end{align}

Note that
\begin{align*}
    D_{\mathcal{H},\delta,\epsilon}=2\Big(\log C(\mathcal{H},\epsilon)+\log{2}/{\delta}\Big)
    =2h\log ({2}/{\delta\epsilon}).
  \end{align*}
  From  \eqref{the-eq-minggd} with $\epsilon=1/N$,
  we have
  \begin{align*}
     R(\hat{f}_i)-R(f_\ast)\leq
     O\left(
     \frac{mh\log\frac{N}{\delta}}{N}+\sqrt{\frac{m R_\ast \log\frac{1}{\delta}}{N}}
     \right).
     %&\frac{Gmh\log(N/\delta)}{N}+\frac{Gmh\log(N/\delta)}{N^2}+\frac{G}{N}+\frac{Lmh\log(N/\delta)}{N}\\
%     &+\frac{Lm\log(1/\delta)}{N}+\sqrt{\frac{GR_\ast m}{N}}.
  \end{align*}
  Substituting the above equation into \eqref{equaiton-strongly-ff} with $\eta=0$,
  we have
  \begin{align*}
     R(\bar{f})-R(f_\ast)&\leq R(\hat{f}_i)-R(f_\ast)\\
     &\leq
     O\left(
     \frac{mh\log\frac{N}{\delta}}{N}+\sqrt{\frac{m R_\ast \log\frac{1}{\delta}}{N}}
     \right).
  \end{align*}
  So, when $m\leq O(N^r)$, we can get,
  \begin{align*}
     R(\bar{f})-R(f_\ast)\leq
     O\left(
     \frac{h\log\frac{N}{\delta}}{N^{1-r}}+\sqrt{\frac{R_\ast \log\frac{1}{\delta}}{N^{1-r}}}
     \right).
     %&\frac{Gmh\log(N/\delta)}{N}+\frac{Gmh\log(N/\delta)}{N^2}+\frac{G}{N}+\frac{Lmh\log(N/\delta)}{N}\\
%     &+\frac{Lm\log(1/\delta)}{N}+\sqrt{\frac{GR_\ast m}{N}}.
  \end{align*}
  If the optimal risk $R_\ast\leq O\left(N^{r-1}\right)$,
  then
  \begin{align*}
    \sqrt{\frac{R_\ast \log\frac{1}{\delta}}{N^{1-r}}}\leq O\left(\frac{1}{N^{1-r}}\right).
  \end{align*}
  Thus, in this case, we have
  \begin{align*}
    R(\bar{f})-R(f_\ast)\leq
     O\left(
        \frac{h\log\frac{N}{\delta}}{N^{1-r}}
     \right).
  \end{align*}
\end{proof}

\bibliographystyle{IEEEtran}
\bibliography{TIT}

\end{document}